\documentclass[12pt]{article}  
\usepackage{graphicx, amsthm, amssymb, xcolor, amsmath, caption, xcolor} 
\usepackage[margin=1in]{geometry}  
\usepackage{hyperref}
\usepackage[numbers]{natbib}  
\usepackage{amsmath, amssymb}
\usepackage{graphicx}
\usepackage{algorithm}
\usepackage{algpseudocode}
\usepackage{amsmath,amsthm,amssymb}
\usepackage{amsthm}
\newtheorem{theorem}{Theorem}[section]

\title{Graph Neural Network-Informed Predictive Flows for Faster Ford-Fulkerson and PAC-Learnability}
\date{}
\author{%
Eleanor Wiesler\textsuperscript{*,1,2}
\and
Trace Baxley\textsuperscript{*,1,2}
}

\begin{document}

\maketitle

\begingroup
\renewcommand{\thefootnote}{\arabic{footnote}}
\setcounter{footnote}{0}
\footnotetext[1]{Harvard Department of Mathematics.}
\footnotetext[2]{Harvard Department of Computer Science.}
\endgroup

\begingroup
\renewcommand{\thefootnote}{\fnsymbol{footnote}}
\setcounter{footnote}{1}
\footnotetext{Joint first authors (equal contribution).}
\endgroup

\section{Abstract}
We propose a novel learning-augmented framework for accelerating max-flow computation and image segmentation by integrating Graph Neural Networks (GNNs) with the Ford–Fulkerson algorithm. Our approach departs from prior work on predictive flow initialization by instead learning edge importance probabilities that guide augmenting path selection. We introduce a Message Passing Graph Neural Network (MPGNN) architecture that jointly learns node and edge embeddings through a mutually dependent update mechanism, enabling the model to capture both global structural context and local flow dynamics such as residual capacity, bottlenecks, and cut structure.

Given an input image, we construct a grid-based flow network with source and sink nodes, extract node and edge features, and perform a single GNN inference to assign probabilities to edges indicating their likelihood of belonging to high-capacity cuts. These probabilities are stored in a priority queue and used to guide a modified Ford–Fulkerson procedure that prioritizes high-value augmenting paths via an adjusted Edmonds–Karp search with bottleneck-aware tie-breaking. This design avoids repeated inference over residual graphs while still leveraging learned structural insights throughout the optimization process.

We further introduce a bidirectional path construction strategy centered on high-probability edges and provide a theoretical framework relating prediction quality to algorithmic efficiency via a weighted permutation distance metric. Our method preserves optimality of the max-flow/min-cut solution while significantly reducing the number of augmentations required in practice. Additionally, we outline a hybrid extension that combines flow warm-starting with edge-priority prediction. This work establishes a foundation for learning-guided combinatorial optimization for  application in efficient image segmentation.
\section{Introduction}

The Ford-Fulkerson algorithm is a classical method for computing the maximum flow in a network from a designated source node to a sink node. In its simplest form, a flow network is a directed graph in which each edge is assigned a capacity, indicating the maximum flow it can carry. A residual graph is constructed during the algorithm’s execution to track remaining capacities and facilitate augmenting path selection. The algorithm proceeds by repeatedly finding augmenting paths in the residual graph and pushing flow along them until no such paths remain, at which point the current flow is guaranteed to be maximal. While elegant and widely used, the algorithm’s speed and performance heavily depends on the choice of starting flow and augmenting paths—making it a prime candidate for informed guidance from learning-based methods.

Recent work by Davies et al \cite{pmlr-v202-davies23b} improves the speed of the classical Ford-Fulkerson algorithm for computing maximum flows by seeding Ford-Fulkerson with predicted flows generated by a PAC-learning framework.  We wish to learn optimal flows with GNNs  to both seed the algorithm with geometry-informed starting predictions and, secondly, use GNN-based prediction to inform edge-level predictions for augmenting path finding. 

Recent work has focused on augmenting such algorithms with machine-learned predictions to improve their performance.
GNNs have been used for combinatorial optimization problems including shortest paths, matching, and routing. Several works show that GNNs can approximate solutions to flow problems with high accuracy, and generalize across graph and network sizes and distributions \cite{gui2023flowx}, \cite{busch2021nf}, \cite{chen2024gflow}, \cite{he2022flow}. Our work seeks to add to the PAC-learning paradigm used in \cite{pmlr-v202-davies23b}. We introduce a variety of theoretical results on the learnability of edge selection functions under the Probably Approximately Correct (PAC) framework. Specifically, we model the problem of selecting high-utility edges (i.e., edges frequently appearing in optimal augmenting paths) as a multi-class classification problem over graph edges, where the target is to label edges as "useful" or "not useful" for augmenting flow. Furthermore, we introduce a variety of GNN-informed algorithms for both warm-start of Ford Fulkerson and augmenting path finding, including extensive work on ideal GNN architectures for such algorithms for application in image segmentation. Together, our work seeks to contribute to a new body of work for improving Ford-Fulkerson speed with learning-augmented algorithms.

\subsection{Ford-Fulkerson}

The Ford-Fulkerson algorithm is a classical primal method for computing the maximum flow in a capacitated network. Given a directed graph $G = (V, E)$ with integral capacities $c : E \rightarrow \mathbb{Z}_{\geq 0}$ and designated source and sink nodes $s, t \in V$, the algorithm incrementally constructs a feasible flow $f$ by identifying augmenting paths in the residual graph $G_f$ and pushing flow along them. The process terminates when no further augmenting path exists, at which point the flow is guaranteed to be maximal by the Max-Flow Min-Cut Theorem. While its correctness and optimality are well established, the practical performance of Ford-Fulkerson depends critically on the order and structure of augmenting paths. This motivates algorithmic enhancements that can warm-start the method or guide path selection heuristics using learned predictions. Recent work (e.g., Davies et al. 2023) formalizes this direction under a PAC-learning framework, seeding Ford-Fulkerson with approximate flows to reduce convergence time. In our work, we extend this paradigm by integrating graph neural network (GNN) models to predict not only warm-start flows but also edge-level priorities for augmenting path construction. The hypothesis of this paper is that these heuristics can improve practical runtime without sacrificing optimality.

\subsection{Image Segmentation with Ford-Fulkerson}

Graph-based image segmentation is a canonical setting for evaluating combinatorial flow algorithms. An image is represented as a grid graph where each pixel corresponds to a node, and edges encode pairwise affinities derived from local intensity differences. Two special nodes—source ($s$) and sink ($t$)—represent the object and background terminals, with seeds specifying hard constraints. Edges from $s$ and $t$ to the pixel nodes are assigned high capacities based on unary priors, and internal edges are weighted using boundary-sensitive functions (e.g., $w_{pq} = C \exp(-\frac{(I_p - I_q)^2}{2\sigma^2})$). Segmenting the image reduces to solving a max-flow problem, with the min-cut delineating the object from the background. While Ford-Fulkerson is not the most advanced method for image segmentation, its application here is illustrative: the segmentation result serves as a concrete visualization of the max-flow solution, and the grid-like structure of image graphs makes them ideal for benchmarking flow algorithms with spatial coherence.

In this paper, we do not aim to advance the state-of-the-art in segmentation quality. Rather, we use the segmentation task as a controlled domain to validate algorithmic enhancements to Ford-Fulkerson. Specifically, we deploy GNN-informed strategies—both warm-start initialization via predicted flows, and edge-prioritized augmenting paths—to reduce the number of augmenting iterations. The segmentation setting provides both structural regularity and visual interpretability, allowing us to empirically measure the impact of learned predictions on flow runtime and algorithmic efficiency.

\subsection{PAC-Learnability in Relation to GNN-Informed Ford Fulkerson}

Given that the warm-start flow algorithm for Ford-Fulkerson used a new prediction algorithm (basing the prediction on total flow), it is important to know if this value is PAC-learnable. Some sets that are not learnable include the set of all Boolean functions -- an extremely large amount of labeled training examples are required to get a sufficient prediction of the specific boolean function. This introduces the idea of PAC-learning: 
\\

A hypothesis class $\mathcal{H}$ is Probably Approximately Correct (PAC) learnable if: For any distribution $D$, any accuracy $\epsilon > 0$, and confidence $\delta > 0$, there exists polynomial $m(\epsilon,\delta)$ such that with $m$ examples, the algorithm outputs $h \in \mathcal{H}$ satisfying with probability $\geq 1-\delta$:

    $$\text{Error}(h) \leq \text{Best possible error} + \epsilon$$

in time $poly(n, \frac{1}{\epsilon}, \frac{1}{\delta})$
\\

Essentially, for algorithms with predictions, we want the object that is the "learned" part of the algorithm to be PAC learnable. With complex objects like graphs, this is not always obvious ans similarly there has never been a algorithms with predictions for predicting edge probability for satisfying some given some $G$ for graphs for flow-problems. In particular, our function requires us to take the $\arg \max$, which adds additional complexity. In determining the PAC-learnability of edge probabilities, one can . Specifically, with a small enough training size, we may discover that less "powerful" algorithms than GNNs may also predict edge probabilities. Throughout our PAC-learning section, we may reference the VC dimension, which is essentially a gauge on how complex the model is by the points the model can shatter. Other dimensions included in this paper have basis in VC dimension but are used for arguments when the hypothesis class does not include Boolean functions.  

\subsection{Our Contributions}

This paper presents a learning-augmented approach to accelerating the Ford–Fulkerson algorithm by leveraging graph neural networks (GNNs) for both initialization and augmenting path selection. Our main contributions are as follows:

\begin{enumerate}
    \item Theoretical results on PAC learnability: We formalize the learnability of edge selection under PAC-learning theory. Specifically, we show that selecting high-utility edges based on graph metrics is PAC-learnable and derive bounds on sample complexity using the Natarajan dimension. We further demonstrate that image grid graphs admit tighter PAC bounds than general graphs, making them an effective domain for theoretical validation.

    \item Algorithm 1 – GCN-based warm start: We develop a Graph Convolutional Network (GCN) architecture to predict edge-level flows from image-derived graphs. These predicted flows are used to initialize the Ford–Fulkerson algorithm. This warm-start reduces the number of augmenting steps required by preemptively saturating bottlenecks in the residual graph. The full model design and training pipeline are provided in Section 3.1.

    \item Algorithm 2 – MPGNN-guided edge scoring: We introduce a Message Passing Graph Neural Network (MPGNN) that learns edge importance scores in residual graphs. These scores capture structural and capacity-based features and are used to guide path selection via modified DFS, prioritizing high-probability edges likely to lie in optimal augmenting paths. The architecture jointly learns node and edge embeddings and is described in Section 3.2.

    \item Algorithm 3 – GNN-assisted Ford–Fulkerson with max-heap prioritization: We propose a full pipeline where edge scores predicted by an MPGNN on the initial residual graph are cached in a max-heap. The Ford–Fulkerson algorithm proceeds by iteratively selecting the highest-scoring edge and assembling an augmenting path through adjusted DFS from the source to the head and from the tail to the sink. This strategy enables informed augmentations across the full flow process without repeated inference, yielding reduced iteration counts and runtime improvements.
    
    \item Implementation and codebase: Despite experimental challenges, we provide a codebase supporting most data preprocessing steps for image-to-graph generation, GNN training, and flow computation. The code infrastructure is designed for extensibility and is made available for future reproduction and evaluation.
\end{enumerate}

\section{PAC-Learnability of Metric-Based Selection of Edges}

Davies et al. \cite{davies2023predictive} uses predictions on graph flows as a seed to optimize the original Ford-Fulkerson algorithm. Moreover, they proved that high quality flows are learnable, which is intended to bolster the idea that flow-based predictions can be used. In our framework, we will show that predictions on edge probability for some property (is within the shortest path, will lie in the min-cut) is PAC-learnable, and that using such a prediction will decrease the run-time with reasonable error. 

Throughout this section we will be using the following notation to describe the 

\begin{itemize}
    \item $V(G), E(G)$ is the set of vertices of a graph $G$
    \item $M(G,e)$ is a "metric", which represents a quality score for some edge $e \in E(G)$
    \item $c^*(G) = \arg \max_{e \in E} M(G,e)$, or the "optimal" edge choice $e$
\end{itemize}

Initially we propose a proof that one may create a model to predict an edge that satisfies some group of $M(G,e)$ threshold. For example, some graphs may have edges at object boundaries that are traversed by exponentially many shortest paths or near-shortest paths. Using features to prioritize these edges allows the algorithm to saturate bottlenecks in residual graphs early - which may allow for fewer path searches. 

\subsection{PAC Learning for an Edge Based on Some Set of Metrics}

\begin{theorem}
Edge selection from a set of graphs $G$ is PAC learnable with $m$ examples, where $m \in O\left(\frac{|E| d \log^2\left(\frac{|E|d}{\epsilon}\right) + \log\left(\frac{1}{\delta}\right)}{\epsilon}\right)$

\end{theorem}

\begin{proof}

Consider the following Hypothesis Class for an instance space of some family of graphs, which we will call $X$.

$$H_E = \Bigl\{\,h_w : G\mapsto \arg\max_{e\in E(G)}w\!\cdot\!\phi(G,e) \Bigm| w\in\mathbb{R}^d\Bigr\}$$

where $\phi(G,e)$ is a vector with dimension $d$ that contains features that should encode the edge properties referenced by $M$.  

Also, where $X$ is the set of directed graphs on $V$, we define such a family where $n \geq \max_{G \in X} |E(G)|$. 

From this, we then extract our feature map for a graph-edge pair, essentially treating our metrics $M_1(G,e), M_2(G,e), \dots M_d(G,e)$ as features:

$$\phi\colon X \times E_{\max} \to \mathbb{R}^d$$

 $$ \varphi(G, e) = \begin{bmatrix}
    M_1(G,e) \\
    M_2(G,e) \\
    \vdots \\
    M_d(G,e) 
  \end{bmatrix}$$

where $E_{max}$ is the set of all potential edges that could exist in our family of graphs $X$. 

The choice of our feature map does not have to be fully defined, but may contain some of the following features, where $e = (u,v)$: 

\begin{itemize}
  \item \(d_s(u)\): length of a shortest path from source \(s\) to \(u\),
  \item \(d_t(v)\): length of a shortest path from \(v\) to sink \(t\),
  \item \(\deg^+(u)\), \(\deg^-(v)\): out-degree of \(u\), in-degree of \(v\),
  \item \(\Gamma^+(u)\): out-neighbors of \(u\),
  \item \(\Gamma^-(v)\): in-neighbors of \(v\)
  \item \( c_r(e)\): edge capacity
\end{itemize} 

Thus, this leads us to define our concept class in the following way: 

\[
\mathcal{F}= \left\{\,f_w : (G,e) \mapsto w \cdot \phi(G,e) \mid w \in \mathbb{R}^d\,\right\}
\]

A key result is the literature on real-valued function classes is the pseudo-dimension of linear functionals. Recall that the psuedo-dimension (Pdim), generalizes the VC dimension (VCdim) to work with functions in the real numbers, and measures the capacity of a class to shatter real-valued targets. 
The theorem introduced by Anthony and Bartlett et al. is shown below: 
\\

Let $\mathcal{F} = \{ f_w : \mathbb{R}^d \to \mathbb{R} \mid f_w(x) = w \cdot x, \ w \in \mathbb{R}^d \}$ be the class of linear functionals on $\mathbb{R}^d$. Then:
\[
\emph{Pdim}(\mathcal{F}) = d
\]
\\

From knowing this we know that the pseudodimension will grow linearly with the number of metrics we include in our feature vector. However, we must consider the $\arg \max$ portion of our hypothesis class. In order to do this, consider the complexity of mutliclass hypothesis classes formed by the function $\arg\max$. The Natarajan dimension ($Ndim$) is related to the pseudodimension over real valued functions as introduced by Shalev-Schwartz et al in lemma 29.5[].

For a class of multiclass predictors, with each hypothesis having $Pdim(H_{bin}) = d$ where bin indicates the predictor for one of the edges:, we have that the following is true: 

$$Ndim(H) \leq 3kd \log (kd)$$

where $k$ is the number defined by some $\arg \max_{i \in [k]} h_i(x)$. Thus, under these circumstances, predicting edges from graphs via linear classifiers gives us a Natarajan dimension of at most $3 |E_{max}| d \log (|E_{max}|d)$ since we are predicting edges (our argument). We will call $|E_{max}|$, $|E|$ for the rest of this proof. 

Finally, we can determine that an algorithm that encodes a collection of edge metrics with feature vectors is PAC learnable from the Multiclass Fundamental Theorem (29.3):

Specifically, with any input errors $\epsilon$, $\delta$ such that for this algorithm, given a set of random samples of size $O(\frac{3 |E| d \log (|E|d) \log (\frac{|E|3 |E| d \log (|E|d)}{\epsilon}) + \log (\frac{1}{\delta})}{\epsilon})$ picked from some probability distribution $D$, will output with probability $1- \delta$ will pick a multiclass function with error less than $\epsilon$. 

The bound on this random sample can be simplified in the following way: 

$$ O\left(\frac{3 |E| d \log(|E|d) \log\left(\frac{3 |E|^2 d \log(|E|d)}{\epsilon}\right) + \log\left(\frac{1}{\delta}\right)}{\epsilon}\right)$$

$$= O\left(\frac{|E| d \log(|E|d) \log\left(\frac{|E|^2 d \log(|E|d)}{\epsilon}\right) + \log\left(\frac{1}{\delta}\right)}{\epsilon}\right)$$ $$O\left(\frac{|E| d \log(|E|d) \left[2\log|E| + \log d + \log\log(|E|d) + \log\left(\frac{1}{\epsilon}\right)\right] + \log\left(\frac{1}{\delta}\right)}{\epsilon}\right)$$ 

$$= O\left(\frac{|E| d \log(|E|d) \log\left(\frac{|E|d}{\epsilon}\right) + \log\left(\frac{1}{\delta}\right)}{\epsilon}\right)$$ 

$$= O\left(\frac{|E| d \log^2\left(\frac{|E|d}{\epsilon}\right) + \log\left(\frac{1}{\delta}\right)}{\epsilon}\right) $$

Since this bound is $poly(|E|, d, \frac{1}{\delta}, \frac{1}{\epsilon})$, the algorithm is PAC-learnable for metric-based edge selection because it is realizable, and can feasibly be used as a prediction for a learned-Ford Fulkerson algorithm.
\end{proof}
 
\subsection{PAC-Learnability with an Image Segmentation-based Family}

Clearly, note that this bound proven above assumes that the graph family contains graphs with the same number of vertices. Within the flow-based warmstart algorithm, notice that in the implimentation of the Ford Fulkerson algorithm, both the number of vertices and the specific edges remain the same. This is because the seeds in the paper's image sequences for both the foreground and background images stay stagnant, which means that the pixel nodes attaching to the source and sink nodes are the same throughout, meaning that each model only works for a single graph topology. With this assumption, the paper assumes that their algorithm will be useful for networking problems which deal with dynamic edge weights only, which sometimes can be a reasonable assumption for networking problems. However, for problems whether there is subtle dynamism with edges, these warm-start models will not work. A great example of where this discrepancy will lie is with image segmentation of similar objects (in our case, the flower data set).

Essentially, we want to show that if the graph family is sufficiently small, then the algorithm to train the model will require less computational resources. 

Suppose we have \(V\), a fixed vertex set, and let
\[
  G_{|V|} = \{\,G_1,\dots,G_N\}
  \subseteq
  \{\text{directed graphs on }V\}
\]
be a family of graphs that all share the same underlying topology except that each \(G_i\) differs from the others only by which \(k\) vertices are attached to the source \(s\) and which \(k\) to the sink \(t\).  We assume \(\max_{i}|E(G_i)|\le m\).
Moreover, given images of size $m \times m$, we know that $|V| = m^2 + 2 = O(m^2)$ and that $|E| = O(4m^2 + 2k)$, and since $k \leq m^2$, we have that $|E| \in O(m^2)$.

Most importantly, note that for some feature vector $\phi(G,e)$, that the features selected may be continuous (like $c_r(e)$ which is a real number in many cases). They also may be discrete. For example, in the case of a feature that represents whether an edge is in a minimum cut of the graph, we can represent that $\phi(G,e)_i \in \{0,1\}$. Moreover, when measuring something like the in neighbor for a node connected to $e$, for the collection of graphs associated with image segmentation cannot be greater than 5 (4 adjacent pixels and a source or sink node). Thus, we can split our feature map into a stagnant features that take in a finite number of values, and the features that take real values. 

Suppose we have the function class:

$$\mathcal{F} = \{ f_w(G,e) = w \cdot \phi(G,e) \mid w \in \mathbb{R}^d \}$$

which is the class of linear scoring functions. Note that the feature vector contains \(d_{\text{stag}}\) stagnant features, representing the discrete space it contains, and takes up\(c\) discrete values). Moreover, the remainder of the features are real-valued; the number of these are denoted by $d_{real}$ 

Because the function is linear, we can decompose the $w \cdot \phi(G,e)$, which is given by:

$$ w \cdot \phi(G,e)= w_{\text{stag}} \cdot \phi_{\text{stag}}(G,e) + w_{\text{real}} \cdot \phi_{\text{real}}(G,e)$$

 We know that the discrete part of the decomposition is bounded above by 
 
 $$w_{\text{stag}} \cdot \phi_{\text{stag}}(G,e) \leq c^{d_{stag}}$$ 
\\

 While we have not completed this proof, the idea is to show that certain constraints of the graph specific to some graph family will allow for a tighter bound on the required examples for PAC-learning. We have yet do this, and are tending toward a negative result on this proof, as for linear functions, having a finite search space allows for less flexibility, and possible a higher VC dimension, meaning a larger requirement for $m$.   

It seems like domain-specific graph families may be advantageous. Thus, since $|E_max|$ is directly dependent on the properties we can introduce the following lemma related to the graphs used for picture segmentation: 

\begin{theorem}
    Let $G_{grid}$ be a family of $m \times m$ grid graphs such that for $G \in G_{grid}$, Grid-based Graphs require less examples $m$ to satisfy PAC-learning constraints than the standard collection on $|V| = n$ vertices 
\end{theorem}

\begin{proof}
 The Natarajan dimension depends on $|E_{max}|$, with grid based graphs, the number of edges is of size $O(4n) = O(n)$, as opposed to $O(n^2)$ for a complete graph. Thus, for analysis using the Multiclass Fundamental Theorem, we have that the sample complexity for $G_grid$ graphs should be of the order $O(sqrt{|E_{max}|})$, which is given by: 

  $$ O\left(\frac{\sqrt{|E|} d \log^2\left(\frac{\sqrt{|E|}d}{\epsilon}\right) + \log\left(\frac{1}{\delta}\right)}{\epsilon}\right) $$
\end{proof}

One could make a similar lemma about other common graph types with non-trivial limits on the number of edges, including trees or planar graphs for example.

\section{GNN-Based Algorithms}
\subsection{Warm-start of Ford Fulkerson with Graph Convolution}
In the context of accelerating the Ford--Fulkerson algorithm for max-flow-based image segmentation, a Graph Convolutional Network (GCN)-based architecture offers a principled and efficient learning framework. The task is to predict edge flows on an image-derived graph that includes not only spatially structured pixel nodes, but also explicit source (object terminal) and sink (background terminal) nodes, which connect to subsets of the pixel graph via terminal edges. These edges encode prior beliefs (from seeded masks or unary potentials) about which pixels belong to foreground and background, respectively.

The input graph $G = (V, E)$ is constructed from an image by treating each pixel as a node and connecting it to its 4- or 8-neighborhood to reflect local intensity similarities. However, to model segmentation via min-cut, two terminal nodes $S$ and $T$ are appended to the graph, representing the object and background terminals. Nodes are connected to $S$ or $T$ via source and sink edges with weights derived from seed maps or learned priors, defining unary costs. The internal edges encode boundary penalties, typically based on gradient magnitude or contrast.

Each node is initialized with a feature vector consisting of spatial coordinates $(x, y)$, grayscale intensity, and a seed encoding (e.g., 0 for neutral, 1 for source, $-1$ for sink). These features are passed through three GCNConv layers with ReLU activations. The GCN layers propagate local feature information across the graph, encoding spatial and semantic structure. We limit the depth to three layers to balance receptive field growth and avoid over-smoothing---a known issue where deep GCNs can produce indistinguishable node embeddings.

After generating hidden embeddings $h_v \in \mathbb{R}^d$ for all nodes (including terminals), we predict edge-level flows. For each edge $(u, v) \in E$, the concatenated node embeddings $[h_u \|\| h_v] \in \mathbb{R}^{2d}$ are passed through an edge MLP, a small feedforward neural network consisting of a hidden layer with ReLU and a final scalar output neuron. This allows the model to learn nonlinear dependencies between adjacent node features and asymmetries in flow patterns, especially important near terminal edges.

The model is trained to minimize the mean squared error between predicted flows and ground-truth flows computed via the exact Ford-Fulkerson algorithm. Because of the flow/cut duality, high-accuracy flow predictions implicitly identify minimum cuts: edges near capacity on the predicted flow correspond to bottlenecks in the segmentation. This allows the GNN to learn both the low-level boundary structure and the high-level topological flow of the segmentation mask. The resulting predicted flows are then used to warm-start Ford-Fulkerson, guiding it toward the bottleneck early and reducing the number of augmenting path iterations required to reach the global optimum.

Once the GNN is trained, the full warm-start procedure is as follows:

\begin{algorithm}[H]
\caption{Warm-starting Ford-Fulkerson with GCN-predicted flow}
\begin{algorithmic}[1]
\State \textbf{Input:} Graph $G = (V, E, c)$, trained GNN model $\mathcal{G}_\theta$
\State Predict flow: $\hat{f} \gets \mathcal{G}_\theta(G)$
\While{$\exists$ edge $e$ with $\hat{f}_e > c_e$}
    \State Clip: $\hat{f}_e \gets c_e$
\EndWhile
\State Set $f \gets \hat{f}$
\State Build residual graph $G_f$, determine excess/deficit sets $A'_f$, $B'_f$
\While{$|A'_f \cup B'_f| > 0$}
    \State Find projection path $p$ from $A'_f$ to $B'_f$ (or to $s/t$ if none exist)
    \State Compute $\mu_p = \min\{\mathrm{ex}_f(w), \mathrm{def}_f(w'), \min_{e \in p} c'_e\}$
    \State Send flow $\mu_p$ down $p$, update $f$, $G_f$, $A'_f$, $B'_f$
\EndWhile
\State Run Ford-Fulkerson on $G$ seeded with $f$ until optimality
\State \textbf{Output:} final flow $f^*$
\end{algorithmic}
\end{algorithm}

\begin{figure}
    \centering
    \includegraphics[width=1\linewidth]{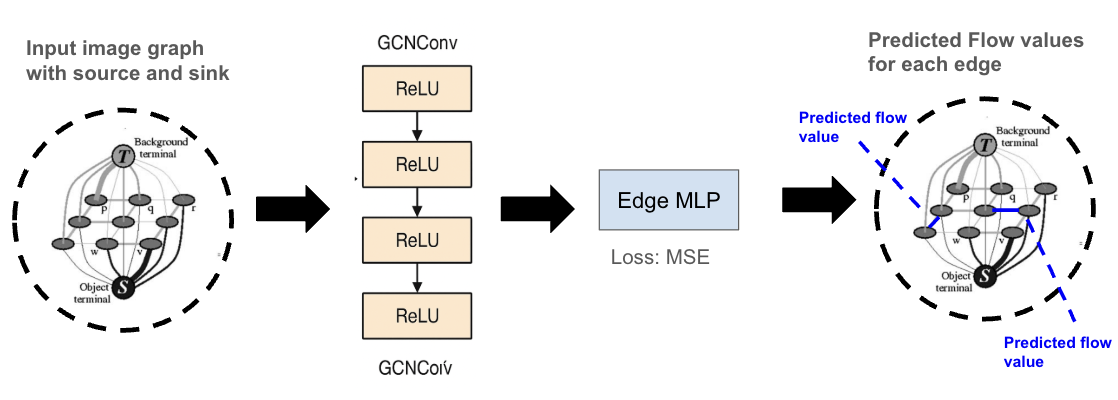}
    \caption{GCN for Algorithm 1: Graph convolution architecture for obbtaining edge-level dlow predictions from image graphs with source and sink nodes.}
    \label{fig:enter-label}
\end{figure}

\begin{figure}
    \centering
    \includegraphics[width=1\linewidth]{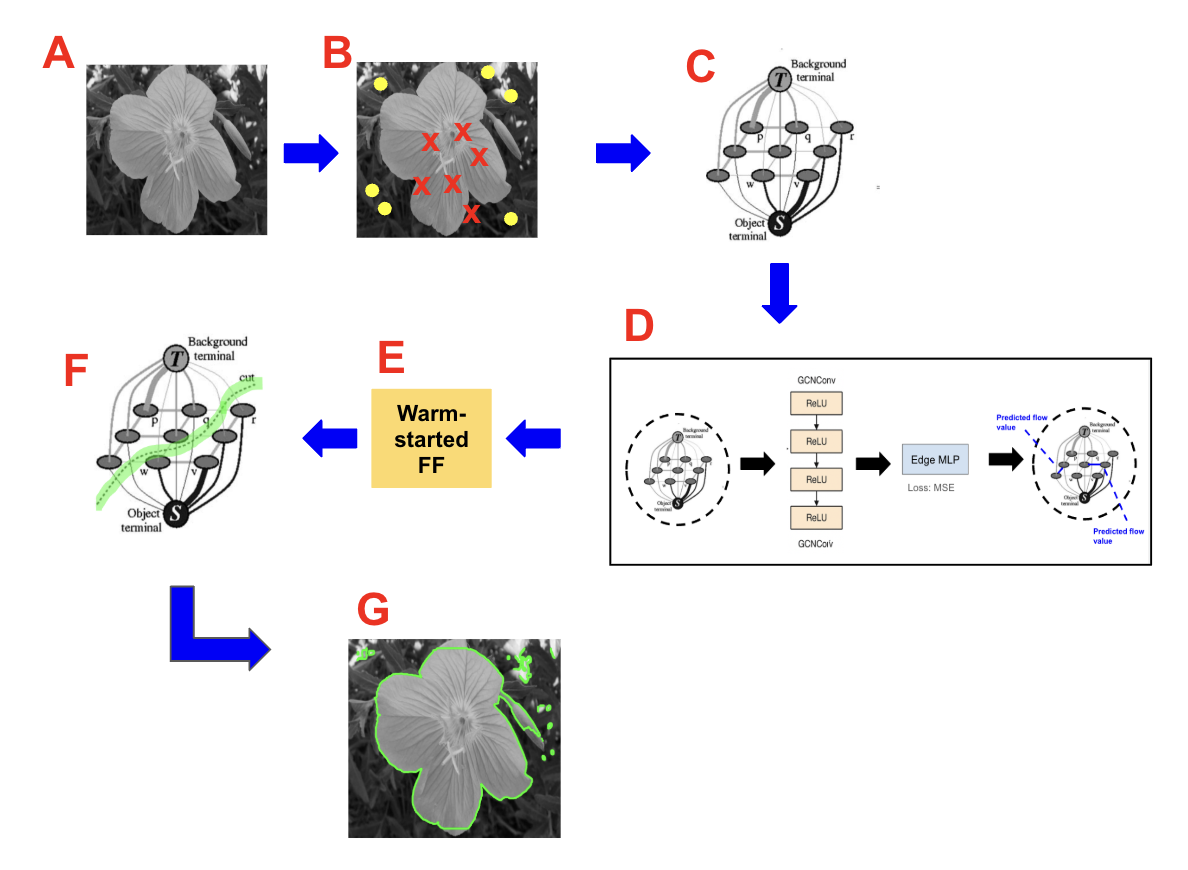}
    \caption{Algorithm 1 Implementation: Training of GCN-based Warm-start of Ford Fulkerson for Image Segmentation. (A) 60 x 60 input object images, (B) foreground and background seeds generated in image preprocessing, (3) grid-like graph representation of image with source and sink nodes and edges added, (D) GCN convolution to generate predicted edge-level flows, (E) Ford-fulkerson step with warm-started edge-level flows, (F) Ford-fulkerson result used for min-cut for segmentation, (E) Resulting segmented image. Note that F and C are adapted from a figure in \cite{image-to-graph-paper}.}
    \label{fig:enter-label}
\end{figure}
\subsection{Augmenting Path Selection With Message Passing Graph Neural Networks}

Let $G = (V, E)$ be a flow network with integral capacities $c \in \mathbb{Z}_{\geq 0}^{|E|}$, and let $f$ denote a feasible flow. The residual graph $G_{\text{res}} = (V, E_{\text{res}})$ is constructed at each iteration of the Ford--Fulkerson algorithm. To accelerate the search for augmenting paths, we introduce a predictive model based on a message-passing graph neural network (MPGNN) that assigns a score $s_{uv} \in [0, 1]$ to each edge $(u,v) \in E_{\text{res}}$, reflecting its learned importance for contributing to a high-capacity augmenting path. These scores guide a modified DFS to prioritize edge exploration and reduce the number of residual graphs required to reach convergence.

Our approach is inspired by the warm-start strategy of Davies et al., which uses a GNN to predict a seed flow that accelerates Ford--Fulkerson. However, rather than predicting flows directly, we aim to predict edge importance probabilities based on structural and flow-related features. Given residual capacity, in- and out-degree, flow history, and edge orientation, a GNN can be trained to estimate the probability that an edge will appear in a high-capacity path or cut. Formally, each edge $e \in E$ is assigned a predicted importance score $p(e)$ by the GNN. These scores are stored in a max-heap $H$ to efficiently identify promising edges.
\begin{figure}
    \centering
    \includegraphics[width=1\linewidth]{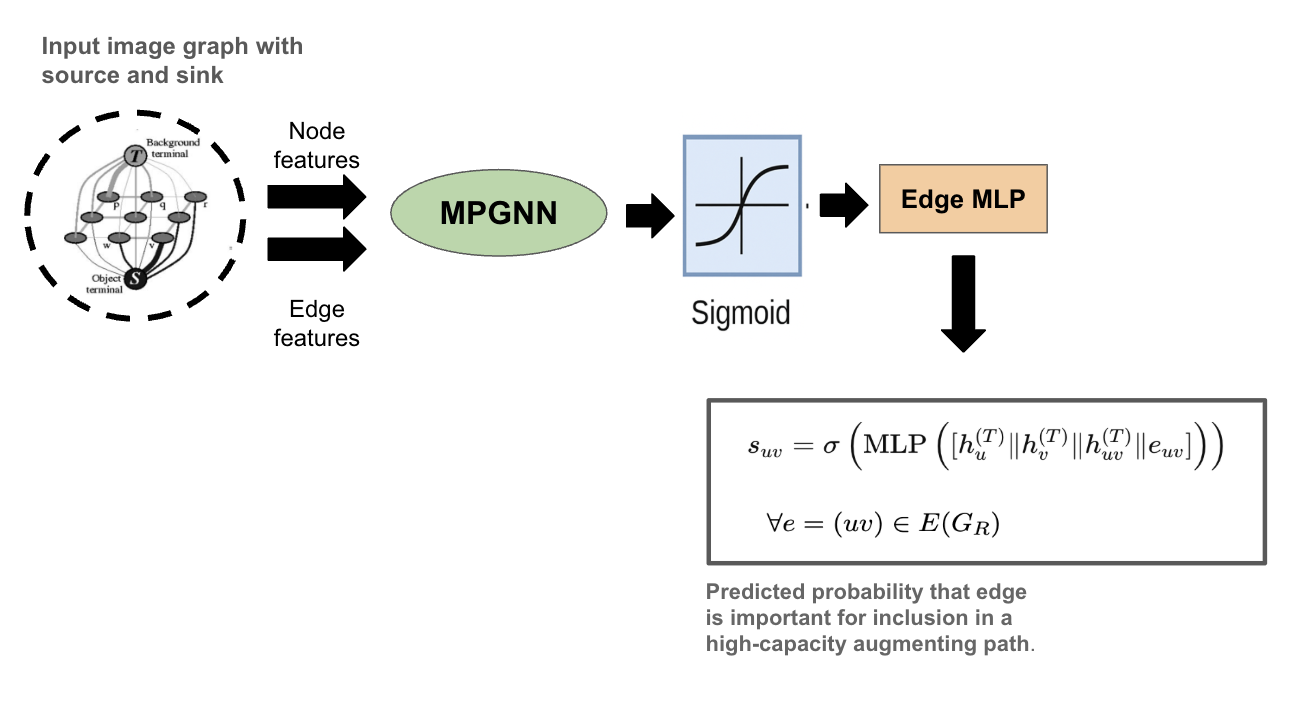}
    \caption{MPGNN Framework for Algorithm 2}
    \label{fig:enter-label}
\end{figure}
\begin{figure}
    \centering
    \includegraphics[width=1\linewidth]{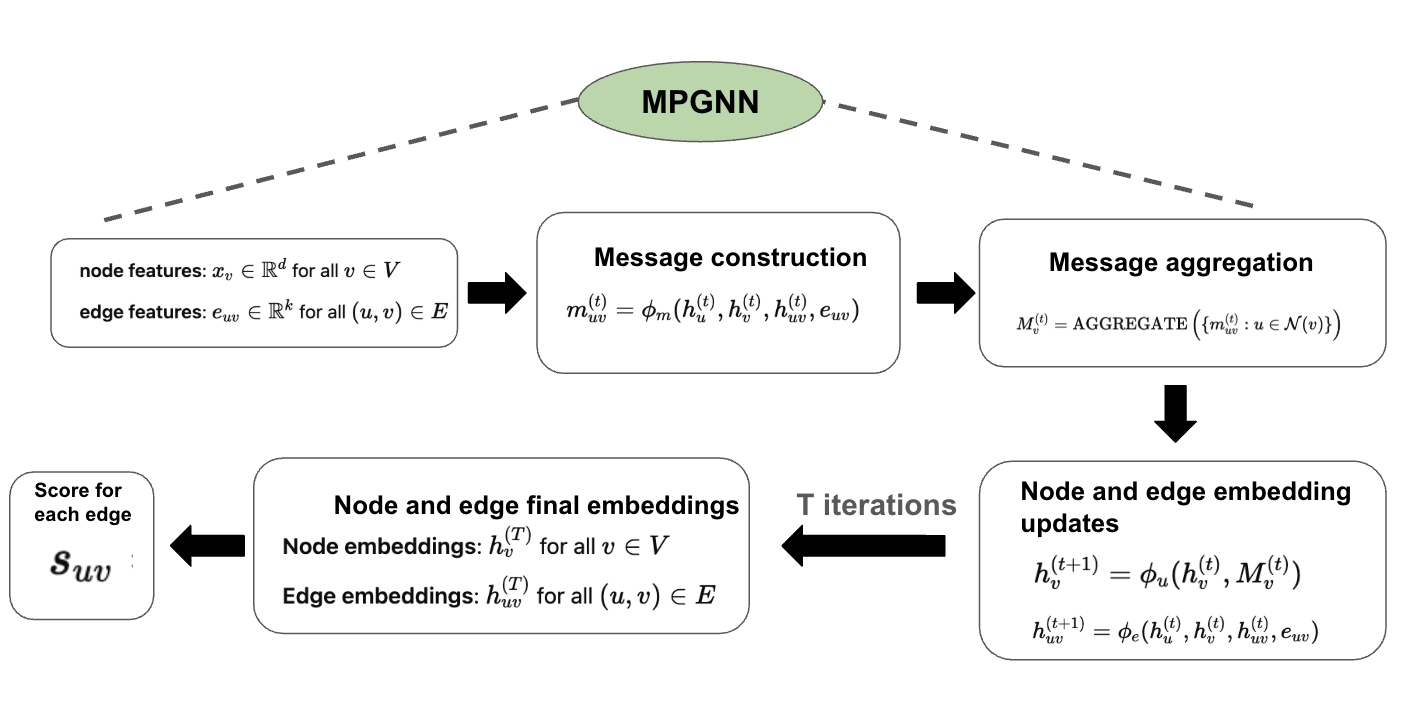}
    \caption{MPGNN-specific steps for node and edge embeddings and $s_{uv}$ generation}
    \label{fig:enter-label}
\end{figure}

The key architectural feature of our approach is the joint learning of both node embeddings and edge embeddings. Each node $v$ is initialized with features $x_v$ such as position, degree, or terminal proximity. Each edge $(u,v)$ is initialized with features $e_{uv}$ including residual capacity $r_{uv}$, original capacity $c_{uv}$, and past flow value $f_{uv}$. Rather than processing these separately, we use a unified MPGNN framework in which **node embeddings evolve in response to edge embeddings**, and **edge embeddings evolve in response to node states**.

At each round $t$, node embeddings $h_v^{(t)}$ and edge embeddings $h_{uv}^{(t)}$ are updated via a mutually interdependent mechanism. First, messages are constructed using a learnable function $\phi_m$ that fuses source and target node embeddings along with edge-specific information:
\[
m_{uv}^{(t)} = \phi_m(h_u^{(t)}, h_v^{(t)}, h_{uv}^{(t)}, e_{uv}).
\]
These messages capture local interactions shaped by both graph topology and capacity features. They are then aggregated at each receiving node to produce an updated node embedding:
\[
h_v^{(t+1)} = \phi_u\left(h_v^{(t)}, \sum_{u \in \mathcal{N}(v)} m_{uv}^{(t)}\right),
\]
where $\phi_u$ may be an MLP or GRU that fuses current and incoming information. Optionally, edge embeddings can be updated as well:
\[
h_{uv}^{(t+1)} = \phi_e(h_{uv}^{(t)}, h_u^{(t)}, h_v^{(t)}, e_{uv}),
\]
where $\phi_e$ may again be an MLP that models directionality and flow context. This back-and-forth coupling allows the network to reason holistically: node embeddings evolve in a flow-aware context, while edge embeddings evolve in a topologically sensitive manner. This mutual attention allows the model to capture critical flow phenomena such as bottlenecks, highly traversed corridors, and near-saturated cuts.

After $T$ rounds, we apply a scoring head that concatenates the final node and edge information:
\[
s_{uv} = \sigma(\text{MLP}([h_u^{(T)} \| h_v^{(T)} \| h_{uv}^{(T)} \| e_{uv}])),
\]
where $\sigma$ is a sigmoid activation producing a probability estimate. The edge score $s_{uv}$ reflects not only static edge attributes but also dynamic relational signals accumulated during message passing. These scores are cached in a priority queue and used to guide an adjusted DFS traversal.

At inference, we select the highest-scoring edge
\[
  e^* = \arg \max_{e \in E_{\text{res}}} p(e),
\]
and compute a full augmenting path by performing two DFS traversals: from the source to $v$ and from $w$ to the sink, where $e^* = (v, w)$. DFS is modified to prioritize edges in descending order of $p(e)$ from the heap $H$. This results in a prioritized path $P = P_1 \cdot e^* \cdot P_2$ along which flow is augmenting. The bottleneck capacity $\delta$ is computed, and $G_{\text{res}}$ is updated. If $e$ is saturated, it is removed from $H$.

\begin{figure}
    \centering
    \includegraphics[width=1\linewidth]{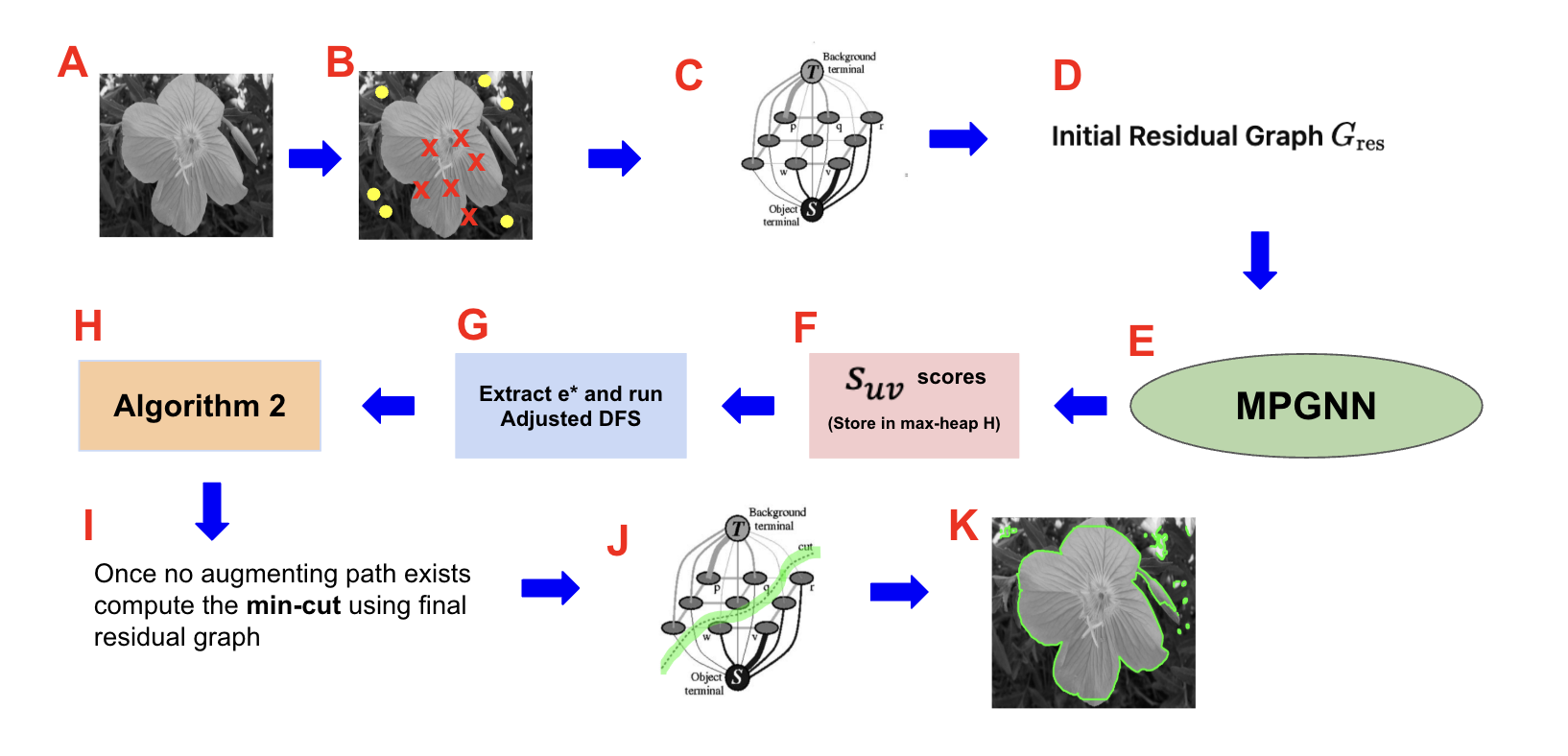}
    \caption{Algorithm 3 Implementation}
    \label{fig:enter-label}
\end{figure}
This mutual learning of node and edge embeddings is essential. Node-only GNNs cannot capture directionality, residual capacity asymmetry, or multi-hop flow structure. Edge-only approaches lack global context. MPGNNs allow both sides of the graph to evolve with awareness of the other—critical for flow-guided search over dynamically evolving residual graphs. The architecture supports non-local flow inference, capacity sensitivity, and can adapt to varying graph sizes and topologies.

Davies et al. uses predictions on graph flows as a seed to optimize the original Ford-Fulkerson algorithm via a predicted seed value for warm-start. However, there are other methods of prediction that could potentially increase flow algorithm efficiency.

The one in which we would like to explore are edge features. Given some residual capacity, in/out degree, flow history, and other factors, Graph Neural Networks can be trained to predict the probability that some edge reaches some residual capacity $\theta$ or is a part of some high capacity cut. The probability that we will try and predict is the one in which $p(e)$ represents the probability that edge $e$ is the edge in the graph $G$ contained in the min-cut with the largest capacity. In this way, we can push the most amount of flow through the graph as possible while being close to guaranteeing that one of the elements in our path will also be in the min-cut, and come closer to completing the iterations in Ford-Folkerson. We can use this information to pick paths which should cut down the number of residual graphs required to terminate the algorithm's process. It would become computationally expensive to use this GNN to predict edge probabilities for every single residual graph (think of the complete graph, for example). However, we can make use of the predictions of our initial graph to optimize path search in the FF algorithm. The following steps are a sketch of such an adjusted algorithm. 

With this prediction, each edge $e$ in some residual graph $G_R(E,V)$ is assigned a probability $p(e)$, which we keep in a max-heap $ (e,p(e)) \in H$. We calculate: 

$$ e^* = \arg \max_{e_i \in E} p(e_i)$$

Given that $e^*= (v,w)$ for $v,w \in V$, we will use and adjusted path search algorithm (ideally one that balances the shortest path and the highest capacity path from the two vertices) to find a path $P_1$ from the source node to $v$ and use the adjusted path-finding algorithm again to find a path $P_2$ from $w$ to the sink node, taking into account the capacity constraints from $P_1$. The total path $P$ is the concatenation of $P_1$, $e^*$, and $P_2$. 

The adjusted path search algorithm is a slight departure of Edmonds-Karp shortest path search, that breaks ties using the highest capacity flow. 

Finally, as usual, we calculate the bottleneck capacity $\delta$ and form the next residual graph. If in the next residual graph an edge is now 0, we remove $(e, p(e))$ from $H$, and start again. However, if a flow uses its backwards edge, we must replace it again in its relative position.  

The psuedocode for this algorithm and our adjusted search algorithm is presented below: 

\begin{algorithm}
\caption{Adjusted Edmonds-Karp with Maximum Bottleneck Tie-Breaking}
\label{alg:adjusted_edmonds_karp}
\begin{algorithmic}[1]
\Require Flow network \( G = (V, E) \), source \( s \), sink \( t \)
\Ensure Maximum flow \( f \)
\State Initialize flow \( f(u, v) \gets 0 \) for all \( (u, v) \in E \)
\While{\textbf{true}}
    \State Run Adjusted BFS with Maximum Bottleneck Tie-Breaking to find augmenting path \( p \) in residual graph \( G_f \) (standard shortest path BFS algorithm with an additional queue for tracking the bottleneck once a path is found)
    \If{no augmenting path found}
        \State \textbf{break}
    \EndIf
    \State Let \( c_f(p) \gets \min\{c_f(e) | e \in p\} \)
    \ForAll{edges \( (u, v) \in p \)}
        \State \( f(u, v) \gets f(u, v) + c_f(p) \)
        \State \( f(v, u) \gets f(v, u) - c_f(p) \)
    \EndFor
\EndWhile
\State \Return \( f \)
\end{algorithmic}
\end{algorithm}

\begin{algorithm}[H]
\caption{GNN-Assisted Ford-Fulkerson (with Adjusted Edmonds-Karp)}
\label{alg:gnn_ff}
\begin{algorithmic}[1]
\State \textbf{Input:} $G_R = (V, E),\ s, t,\ p(e)\ \forall e \in E$
\State $H \gets \text{max-heap of } (e, p(e))$
\State $F \gets 0$
\While{$\exists$ augmenting path in $G_R$}
    \State $e^* = \arg\max_{e \in E} p(e),\ e^* = (v,w)$
    \State $P_1 \gets \text{Adjusted-EdmondsKarp}(s, v, H)$
    \State $P_2 \gets \text{Adjusted-Edmonds-Karp}(w, t, H)$
    \If{$P_1 \neq \emptyset\ \land\ P_2 \neq \emptyset$}
        \State $P \gets P_1 \cup \{e^*\} \cup P_2$
        \State $\delta \gets \min_{e \in P} c(e)$
        \State Augment flow along $P$ by $\delta$
        \State Update $G_R$
    \EndIf
    \State Remove $(e, p(e)) \in H$ if $e \notin G_R$
\EndWhile
\State \Return $F$
\end{algorithmic}
\end{algorithm}

We apply this architecture once to the initial residual graph and use the learned edge probabilities to guide the entire FF process via Algorithm~\ref{alg:gnn_ff} and the adjusted DFS Algorithm. This design avoids recomputing predictions at each step, which would be computationally expensive.

The resulting pipeline is as follows: an image is converted to a flow graph, node and edge features are extracted, a single MPGNN inference produces edge scores, and Ford--Fulkerson proceeds with GNN-guided augmenting paths. The algorithm yields a segmentation based on the min-cut, with substantially fewer augmentations required compared to uninformed DFS, while preserving optimality.\cite{shalev2014understanding}
\section{Future Directions}
\subsection{Suggested Further Research}

As discussed in our introduction, this paper is a stepping stone to extensive further experimental and theoretical work.

\begin{enumerate}
    \item \textbf{Ford--Fulkerson warm-start experimentation with a GCN framework.}
    \begin{itemize}
        \item Proposed experimental design: start by processing only 500 images at $60\times 60$ pixels, generate foreground/background seeds, and develop a simple test metric to verify that source and sink nodes are added correctly.
    \end{itemize}

    \item \textbf{Ford--Fulkerson augmenting-path experimentation with an MPGNN framework.}
    \begin{itemize}
        \item Proposed experimental design: start by processing only 500 images at $60\times 60$ pixels, generate foreground/background seeds, and develop a simple test metric to verify that source and sink nodes are added correctly.
    \end{itemize}

    \item \textbf{Proof of Algorithm~3 Runtime.}
    \begin{itemize}
        \item  Typically, Edmonds--Karp has worst-case runtime $O(|V||E|^2)$. If the algorithm predicts with some error $\epsilon$, the idea is that under perfect conditions (a perfect predictor with $\epsilon = 0$) we can finish in $k$ iterations, where $k$ is the size of the min-cut, while under an adversarial predictor the algorithm is equivalent to running Edmonds--Karp as normal. A next step would be to derive the expected runtime as a function of prediction error; potential back-edges in the residual graph make this worthy of further research.
    \end{itemize}
\end{enumerate}

\subsection{Theoretical Next Steps}

We want to prove the theoretical accuracy bounds required for the model in order for this Bi-directional Ford Fulkerson algorithm (Algorithm 2) with prediction to achieve max-flow in a reduced time-complexity. Even though the output of our GNN model given a graph and an edge as input is a real probability, what the model is effectively measuring is the ordering of the edges based on $p$ (as seen in the heap $H$). Thus, given some ground truth permutation $\sigma'$, we intend to show that given some predicted permutation $\hat{\sigma}$ from the GNN, that if the the "distance" between them is sufficiently small, then the adjusted FF algorithm will require less iterations, some number defined by this "distance". 

There are several choices of metrics to measure distance between but the most obvious is the Cayley distance, \cite{de2012complexity} which is defined as follows, with $\sigma,\hat{\sigma}$ as two permutations. 

$$d_C(\sigma,\hat{\sigma}) = n - c(\sigma^{-1}\hat{\sigma})$$

 where $c(\sigma)$ being defined as the number of cycles of a permutation and $d_C(\sigma,\hat{\sigma})$ is the minimum number of transpositions required to transform $\sigma$ into $\hat{\sigma}$.

For our purposes, edges at the top of the ordering, the ones that have the highest predicted probability, carry more of an importance than those further down the line (as those near the top are more likely to become $e^*$ in some residual graph). 
In our context, edges at the top of the ordering (i.e., with the highest predicted probabilities) are more important than those further down. To capture this, we define a weighted Cayley distance as follows.

Let $w : \{1,2,\ldots, n\} \rightarrow \mathbb{R}_{>0}$ be a weight function satisfying
\[
w(1) \geq w(2) \geq \cdots \geq w(n).
\]
If we represent a permutation $\pi \in S_n$ in one-line notation, deviations in the top positions should incur a higher penalty. Thus, we define the weighted Cayley distance between the ground truth permutation $\sigma$ and the predicted permutation $\hat{\sigma}$ as follows.

$$d_{WC}(\sigma,\hat{\sigma}) = \min \left\{ \sum_{j=1}^{k} w(i_j) : \; \hat{\sigma} = \tau_{i_k}\cdots\tau_{i_1}\circ \sigma \right\}$$

where some $\tau_{i}$ is a transposition on the $i$th position and the minimum is taken over all possible sequences of transpositions that transform $\sigma$ into $\hat{\sigma}$.

Our goal is to show that if this weighted Cayley distance $d_{WC}(\sigma,\hat{\sigma})$ is bounded by some sufficiently small value, then the predicted ordering of edges is close to the ground truth, which then implies that the paths selected by the adjusted Ford-Fulkerson algorithm are nearly optimal. Thus, we will formalize the relationship between the weighted permutation distance and the quality of the augmenting paths chosen, derive bounds on the improvement in the flow augmentation per iteration as a function of this distance metric, and finally then demonstrate that the total number of iterations required to achieve our max-flow value is reduced compared to the standard Ford-Fulkerson algorithm, and deduce the time complexity from this analysis. Additionally, we propose further experimental and theoretical work on the following combined algorithm: 

Algorithm 4 integrates two types of GNN predictions. First, a GNN is used to generate a preliminary flow which is then projected to feasibility. Second, a separate GNN predicts edge priorities in the residual graph, allowing for more informed augmenting path selection in the Ford-Fulkerson procedure. This hybrid approach is designed to benefit from both faster initialization and smarter updates, ideally reducing total iterations required.

\begin{algorithm}[H]
\caption{Warm-Started GNN-Assisted Ford-Fulkerson}
\label{alg:combined}
\begin{algorithmic}[1]
\State \textbf{Input:} Graph $G = (V, E, c)$, trained flow predictor $\mathcal{G}_\theta$, edge-priority predictor $\mathcal{P}_\phi$
\State Predict flow: $\hat{f} \gets \mathcal{G}_\theta(G)$
\While{$\exists$ edge $e$ with $\hat{f}_e > c_e$}
    \State Clip: $\hat{f}_e \gets c_e$
\EndWhile
\State Set $f \gets \hat{f}$
\State Build residual graph $G_f$, determine excess/deficit sets $A'_f$, $B'_f$
\While{$|A'_f \cup B'_f| > 0$}
    \State Find projection path $p$ from $A'_f$ to $B'_f$ (or to $s/t$ if none exist)
    \State Compute $\mu_p = \min\{\mathrm{ex}_f(w), \mathrm{def}_f(w'), \min_{e \in p} c'_e\}$
    \State Send flow $\mu_p$ down $p$, update $f$, $G_f$, $A'_f$, $B'_f$
\EndWhile

\State Predict edge probabilities: $p(e) \gets \mathcal{P}_\phi(G_f)$ for all $e \in E$
\State Initialize max-heap $H$ with tuples $(e, p(e))$
\While{$\exists$ augmenting path in $G_f$}
    \State $e^* = \arg\max_{e \in G_f} p(e),\ e^* = (v,w)$
    \State $P_1 \gets \text{AdjustedDFS}(s, v, H)$
    \State $P_2 \gets \text{AdjustedDFS}(w, t, H)$
    \If{$P_1 \neq \emptyset\ \land\ P_2 \neq \emptyset$}
        \State $P \gets P_1 \cup \{e^*\} \cup P_2$
        \State $\delta \gets \min_{e \in P} c(e)$
        \State Augment flow $f$ along $P$ by $\delta$
        \State Update residual graph $G_f$ and probabilities $p(e)$
    \EndIf
    \State Remove $(e, p(e))$ from $H$ if $e \notin G_f$
\EndWhile
\State \Return final flow $f$
\end{algorithmic}
\end{algorithm}
A next theoretical step would be to integrate this algorithm, and perform rigorous testing compared with traditional Ford-Fulkerson. Moreover, it would be useful to formalize the algorithm sketch for the run-time of the algorithm with edge probability predictions

\section{Conclusion}

We introduced a set of novel algorithms in the space of algorithms with predictions which utilized predictions on edge probabilities for graph families. In order to determine if this was a valid prediction, just as in the warm-start algorithm with predictive flows, we proved the PAC-learnability of edge probabilities for some general metrics. Then, we went on to define a bidirectional algorithm that takes in predictions for sufficient edges, and traces paths from the predicted edge to the source and the sink respectively using an adjusted version of Edmonds-Karp. Finally, we introduced GNN-frameworks tailored to the warm-start algorithm for flows, which was a Graph Convolutional Network, and a framework tailored to path selection with learned edge probabilities, which was a message-passing graph neural network.  We encourage experimental analysis of these results as a next step and further theoretical research on algorithmic runtimes.

\bibliographystyle{plainnat}  
\bibliography{ref}

\end{document}